\documentclass[10pt,twocolumn,letterpaper]{article}

\usepackage{cvpr}
\usepackage{times}
\usepackage{epsfig}
\usepackage{graphicx}
\usepackage{amsmath}
\usepackage{amssymb}
\usepackage{amsthm}
\usepackage{color}
\usepackage{algorithm}
\usepackage{algorithmic}
\usepackage{subcaption}
\usepackage{booktabs}
\usepackage{authblk}
\usepackage{color}

\cvprfinalcopy

\DeclareMathOperator*{\argmax}{arg\,max} 
\DeclareMathOperator*{\kmax}{k-max}
\DeclareMathOperator*{\merge}{merge}

\newtheorem{theorem}{Theorem}[section]
\newtheorem{lemma}{Lemma}
\newcommand\blfootnote[1]{%
  \begingroup
  \renewcommand\thefootnote{}\footnote{#1}%
  \addtocounter{footnote}{-1}%
  \endgroup
}

\begin{document}

\title{Temporal Action Localization by Structured Maximal Sums}

\author[1,2]{Zehuan Yuan}
\author[2]{Jonathan C. Stroud}
\author[1]{Tong Lu}
\author[2]{Jia Deng}
\affil[1]{State Key Laboratory for Novel Software Technology, Nanjing University, China}
\affil[2]{University of Michigan, Ann Arbor}

\maketitle
\begin{figure*}[t!]
\centering
\includegraphics[width=0.8\textwidth]{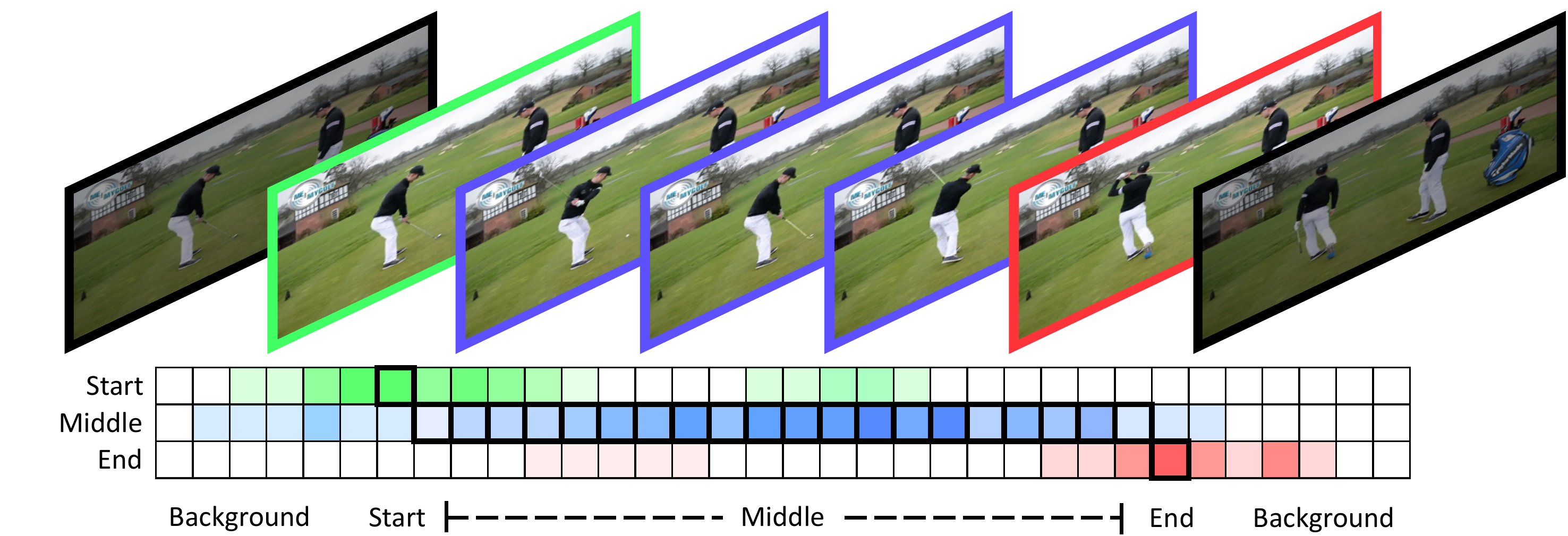}
\vspace{-2mm}
\caption{Temporal evolution of a \emph{golf-swing} action. Our system
  explicitly models evolution as a single \emph{start} frame (green),
  followed by many \emph{middle} frames (blue), and a single
  \emph{end} frame (red).}
\label{fig:evolution}
\end{figure*}
\begin{abstract}
  We address the problem of temporal action localization in videos. We
  pose action localization as a structured prediction over
  arbitrary-length temporal windows, where each window is scored as
  the sum of frame-wise classification scores. Additionally, our model
  classifies the start, middle, and end of each action as separate
  components, allowing our system to explicitly model each action's
  temporal evolution and take advantage of informative temporal
  dependencies present in this structure. In this framework, we
  localize actions by searching for the \emph{structured maximal sum},
  a problem for which we develop a novel, provably-efficient
  algorithmic solution. The frame-wise classification scores are
  computed using features from a deep Convolutional Neural Network
  (CNN), which are trained end-to-end to directly optimize for a novel
  structured objective. We evaluate our system on the THUMOS '14
  action detection benchmark and achieve competitive performance.
\end{abstract}

\vspace{-5mm}

\section{Introduction}
\label{sec:intro}

\blfootnote{This work was done while Zehuan Yuan was a visiting student at the
  University of Michigan.} In temporal action localization, we are
given a video and aim to detect \emph{if} and \emph{when} a particular
action occurs. Specifically, we answer three questions -- \emph{``is
  there an action in the video?''}, \emph{``when does the action
  start?''}, and \emph{``when does the action end?''}. By automating
this process, we can enable people to efficiently search through the
millions of hours of video data which are generated each and every
day. However, this remains a challenging problem for several reasons.
Crucially, actions have inherent temporal structure, so we require a
representation that can model the \emph{temporal evolution} of actions
in addition to their instantaneous spatial appearance. Previous
methods have either failed to model temporal evolution, or done so at
significant computational
cost~\cite{DBLP:conf/iccv/OneataVS13,DBLP:conf/nips/SimonyanZ14,
  DBLP:journals/corr/ShouWC16}. High computational cost is a
significant problem for these methods, because in many practical
applications, the videos of interest may be arbitrarily long, and
methods must gracefully scale to videos that last hours (\emph{e.g.}
movies, web videos) or even days (\emph{e.g.} security footage,
first-person vision). Finally, extracting powerful features for
detecting actions in videos remains an unsolved challenge.

To overcome these challenges, we propose a method that directly models
the temporal evolution of actions, and we develop a provably-efficient
algorithm to perform localization in this framework. Our temporal
evolution framework is based on the observation that all actions have
a \emph{start}, \emph{middle}, and \emph{end}, and that these
components each have distinct patterns of appearance and motion. We
hypothesize that by localizing these three action parts separately, we
can significantly improve localization performance by enforcing
consistent structure in their ordering. Specifically, we model an
action as a \emph{temporal window}~-- a variable-length sequence of
video frames -- and we assume that each temporal window begins with a
single \emph{start} frame, followed by one or more \emph{middle}
frames, and finally a single \emph{end} frame
(Figure~\ref{fig:evolution}). We otherwise impose no restrictive
constraints on the temporal sequence of each action. In doing this, we
recover just enough temporal information to take advantage of the
inherent structure present in each action, without requiring any
additional annotations or making unrealistic assumptions about the
composition of actions.

At test time, we localize actions by searching for the
\emph{structured maximal sum} -- the sequence of start, middle, and
end frames which has the highest sum of corresponding frame-wise
confidence scores. Solving this problem is non-trivial, as it
na\"{i}vely requires a search over a quadratic number of possible
start-end pairs. However, in Section~\ref{sec:detection}, we propose a
novel dynamic programming algorithm which provably finds the top-$k$
structured maximal sums for a video of arbitrary length. We prove that
this algorithm is efficient, and specifically we show that it finds
the structured maximal sum in linear time. Our solution is related to
that of the well-studied \emph{$k$-maximal sums problem}, for which
similar efficient algorithms exist \cite{DBLP:books/daglib/0067352}.
Our structured maximal sum algorithm enables us to gracefully scale
localization to arbitrary-length untrimmed videos, while
simultaneously encoding the temporal evolution of each action.

We classify the three action components separately using powerful
discriminative features from two-stream Convolutional Neural Networks
(CNNs)~\cite{DBLP:conf/nips/SimonyanZ14}. In
Section~\ref{sec:structuralsvm}, we train the entire system
end-to-end, using a novel structured loss function. We train and
evaluate our approach on the THUMOS'14 challenge
dataset~\cite{THUMOS15} in Section~\ref{sec:experiments} and achieve
competitive results.

Our primary contribution is a framework that allows us to model the
temporal evolution of actions without sacrificing efficient temporal
localization. Crucial to our framework is a novel, provably-efficient
algorithm, which computes the structured maximal sum in linear time.
We achieve competitive results on action detection baselines, and
present a number of ablation studies to demonstrate the contributions
of each component of our pipeline.

\section{Related work}
\label{sec:related}

Temporal action localization in videos is an active area of research,
and much recent progress has been facilitated by an abundance of
benchmark datasets and competitions which focus on temporal
localization, including the THUMOS~\cite{THUMOS15} and
ActivityNet~\cite{DBLP:conf/cvpr/HeilbronEGN15} challenges. Most prior
approaches have fallen into one of two categories: \emph{sliding
  window} and \emph{framewise} classification. In this section we will
outline the major contributions of these approaches.

\noindent \textbf{Sliding Window.} Many leading approaches for
temporal localization apply classifiers to fixed-width windows that
are scanned across each video. These approaches have the advantage
that they are able to consider contextual information and temporal
dependencies in their classifications, up to the size of the temporal
windows. Oneata {\it et al.} \cite{DBLP:conf/iccv/OneataVS13}, the
winners of the THUMOS '14 localization competition, used sliding
window classifiers applied to fisher vector representations of
improved dense trajectories features \cite{DBLP:conf/iccv/WangS13a}.
Wang {\it et al.} took a similar sliding-window approach and came in
second place in the same competition~\cite{wang2014action}. Recently,
Shou {\it et al.} proposed a sliding 3D Convolutional Neural Network
for localization, in favor of the hand-designed features of previous
methods \cite{DBLP:journals/corr/ShouWC16}. In a similar vein to our
temporal evolution model, Gaidon {\it et al.} used sliding window
classifiers to locate action parts (actoms) from hand-crafted features
\cite{DBLP:journals/pami/GaidonHS13}. The key distinction between
their sequence model and ours is that their action parts are specific
to each individual action, and must be chosen and labeled manually,
while ours uses the same parts for each action, and requires no
additional annotations.

Most sliding window approaches are applied at multiple window sizes to
account for variation in temporal scale. This leads to significant
redundant computation and makes these methods expensive to scale to
long videos. However, their success in competitions like THUMOS
demonstrate that the contextual information afforded by sliding window
methods is important for accurate localization.

\noindent \textbf{Frame-wise Classification.} Another class of popular
approaches apply classifiers to each individual frame to detect the
presence or absence of an action. Action windows are then aggregated
during post-processing, using simple non-maximum suppression or more
complex sequence models. Singh~{\it et al.}~\cite{singh2016untrimmed}
achieved competitive performance in the ActivityNet
challenge~\cite{DBLP:conf/cvpr/HeilbronEGN15} using a frame-wise
classifier to propose action locations, aggregated together by
minimizing a loss that encourages piecewise smoothness in the
detections~\cite{singh2016untrimmed}. Sun~{\it et al.} and Singh~{\it
  et al.} applied frame-wise CNN feature detectors, connected by
recurrent LSTM modules~\cite{DBLP:conf/mm/SunSSN15,singhmulti}.
Richard~{\it et al.} adopt language models applied to traditional
motion features~\cite{DBLP:journals/cvpr/Alexander16}.

While each of these methods are able to incorporate some temporal
context in post-processing, they each either rely on hand-designed
frame-level features or optimize some frame-level loss. Our approach,
by contrast, is trained end-to-end, and directly optimizes a
structured loss over temporal action windows, allowing it to learn
features that facilitate accurate action localization.

\noindent \textbf{Other Approaches.} Max-margin losses have been used
to detect actions in an online setting~\cite{hoai2014max} and from 3D
video features~\cite{wei2013concurrent}, but not in an end-to-end
pipeline. Many works have taken other approaches to model the temporal
structure present in actions~\cite{tang2012learning}. Recently,
Yeung~{\it et al.}~\cite{DBLP:journals/corr/YeungRMF15} proposed using
reinforcement learning to actively search for informative frames in a
video before directly regressing the start and end points of each
action. Their approach is efficient, in that it only needs to observe a
few frames before making each prediction, but it does not aggregate
information over the entire video to achieve the best performance.

The work that is perhaps most related to ours is
\cite{DBLP:conf/nips/TranY12}, which approaches action localization as
a structured prediction over spatio-temporal paths through a video,
utilizing the max-path algorithm of \cite{DBLP:conf/cvpr/TranY11} to
perform efficient inference. Their method is capable of performing
both spatial and temporal localization jointly, and similarly uses a
max-margin structured regression to learn frame-wise classification
scores. Our method, however, has the advantage of modeling the
temporal evolution of actions, and utilizes powerful CNN features,
which we train end-to-end.

\section{Localization as Structured Prediction}
\label{sec:formulation}

Suppose we are given a video
$v=\{x_1,x_2,\cdots , x_n\} \in \mathcal{V}$, where $x_t$ denotes the
frame at timestep $t$, and $n$ is the total number of frames in the
video. We define a \emph{temporal window} to be a contiguous segment
of video frames $y=\{x_s,x_{s+1},\cdots,x_e\} \in \mathcal{Y}$, where
$s$ and $e$ are the indices of the start and end frames, respectively,
and $1\leqslant s \leqslant e\leqslant n$. Furthermore, suppose that
each frame has a real-valued \emph{frame-wise score}
$f(x_t) \in \mathbb{R}$ which can be positive or negative, which
indicates our confidence in frame $x$ belonging to an instance of a
particular action class. Note that, for convenience, we express
$f(x_t)$ as a function of only a single frame $x_t$, while in practice
$f$ may depend on features extracted from the entire video. For a
video and corresponding temporal window, we define the
\emph{confidence score}
$F: \mathcal{V} \times \mathcal{Y} \mapsto \mathbb{R}$ as the sum of
framewise scores, that is, $F(v, y) = \sum_{t=s}^{e}f(x_t) \ $. The
predicted temporal window for video $v$ is the one that maximizes the
confidence score, in particular,
$\hat{y} = \argmax_{y\in \mathcal{Y}} \ F(v, y) \ $.

Na\"ively, by searching over all possible start- and end-point pairs,
this maximization requires a search over a space quadratic in the
number of frames. For long videos, this is impractical. However,
because $F$ is decomposable into frame-wise scores, we can pose this
as the classic \emph{maximum sum} problem
\cite{DBLP:books/daglib/0067352}, for which there exists an
$\mathcal{O}(n)$-time solution~\cite{DBLP:conf/ispan/BaeT04}. In
practical settings, we may have multiple action instances in a single
video. Finding the $k$-best windows can similarly be posed as a
$k$-maximal sums problem, for which there exists a
$\mathcal{O}(n+k)$-time solution~\cite{brodal2007linear}. In the
following section, we model more complex temporal dependencies.

\subsection{Temporal Evolution Model}
\label{sec:evolution}

Actions vary greatly in their appearance and motion characteristics
over time. By explicitly modelling the \emph{temporal evolution} of an
action, we can take advantage of this inherent temporal structure. In
particular, we notice that the frames at the start- and end-points of
an action instance tend to vary greatly in appearance, as the actor
will often change position (as in \emph{basketball-dunk}) or pose (as
in \emph{golf-swing}) over the course of the action. Additionally,
frames in the middle of the action instance tend to have different
motion characteristics than the start- and end-points as the actor
performs complex body movements.

The start and end of an action are of particular importance in
temporal localization, as they define the boundaries of a single
action instance. In order to encourage precise localization, we
explicitly model each action as a single \emph{start} frame, followed
by an arbitrary-length series of \emph{middle} frames, and finally a
single \emph{end} frame. Suppose we have separate signed frame-wise
confidence scores $f^s(x)$, $f^m(x)$, and $f^e(x)$, for the
\emph{start}, \emph{middle}, and \emph{end} components, respectively.
Using this new formulation, we can rewrite the confidence score
$F(v, y)$ for a video $v$ as
\begin{equation}
\label{equa:confidence2}
F(v, y)=\lambda_sf^s(x_s)+\lambda_m\sum_{t=s+1}^{e-1}f^m(x_t)+\lambda_ef^e(x_e)\\
\end{equation}
where $\lambda_s$, $\lambda_m$ and $\lambda_e$ are parameters that
specify the relative importance of each action part. In our
experiments, we set $\lambda_s =\lambda_m = \lambda_e = 1$ except
where stated otherwise.

This generalization comes with a number of advantages over the
single-class confidence score without temporal evolution. First, we
are penalized heavily for detections that fail to find good matchings
for the start and end frames. This enforces temporal consistency, as
the best detections will be those that successfully match each of the
three components in their correct order. This resistance to illogical
matchings gives us robustness to variance in the frame-wise scores.
This makes us less likely to merge consecutive or partial instances of
an action into a single detection, and encourages the detector to
stretch each detection to the full extent of the action instance,
preventing over- and under-segmentation. Finally, start and end labels
are readily available from existing temporal action annotations,
meaning that we require no additional training data. Finally, since
every action has a \emph{start}, \emph{middle}, and \emph{end}, this
formulation makes no restrictive assumptions about the structure of
complex actions.

\section{Structured Maximal Sums}
\label{sec:detection}

Given frame-wise scores $f^s$, $f^m$, and $f^e$, and a video $v$, our
goal is to detect all instances of a particular action. We represent
these detections as the top-$k$ temporal windows, as ordered by their
confidence scores in Equation \ref{equa:confidence2}. In Section
\ref{sec:formulation}, we showed how, without temporal evolution,
localization can be framed as a $k$-maximal sums problem
\cite{DBLP:conf/ispan/BaeT04}. However, our formulation introduces
additional challenges, as we now need to compute the top-$k$
\emph{structured} maximal sums (Figure~\ref{fig:evolution2}), which
previous work does not address. To solve this problem, we introduce
the Structured Maximal Sums (SMS) algorithm
(Algorithm~\ref{alg:smekoptimal}), which efficiently finds the top-$k$
structured maximal sums.

The Structured Maximal Sums algorithm makes a single pass through the
video, maintaining the value of the $K$-best windows found so far in
the list $\mathrm{kmax}[:]$, which is kept in sorted order. It also
keeps track of the values of the $K$-best \emph{incomplete temporal
  windows} that end at frame $i$ in $\mathrm{rmax}[:]$, that is, the
windows that end at $i$ but do not include an endpoint $f^e[i]$. We
now prove the correctness of the SMS algorithm.

For clarity, we first introduce the following notation. We assume that
all frame-wise classifier scores are precomputed, and are contained in
ordered lists $f^s[1 \cdots n]$, $f^m[1 \cdots n]$, and
$f^e[1 \cdots n]$, and the shorthand $[:]$ refers to all elements in a
list simultaneously. Similarly, we denote adding a value to each
member of a list as $f[:] + n$. We denote the operation of inserting
an item $s$ into a sorted list $\mathrm{kmax}$ as
$\merge(s, \mathrm{kmax})$. We denote the $k$-th maximum value of a
function $g$ over a discrete space $\mathcal{X}$ as
$\kmax_{x \in \mathcal{X}} g(x)$.

\begin{figure}[t]
\centering
\includegraphics[width=1.0\linewidth]{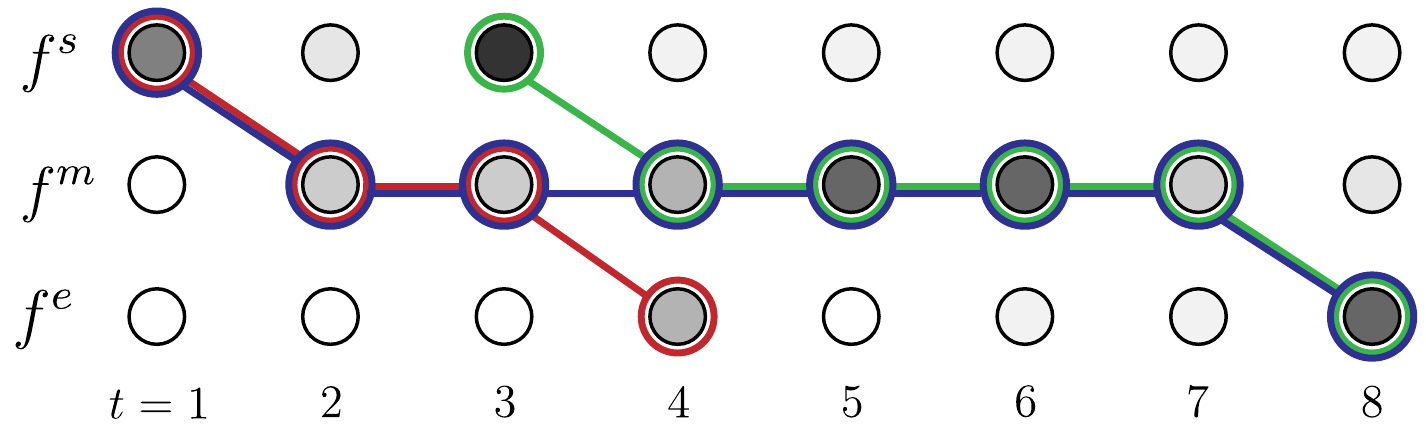}
\caption{A depiction of the structured maximal sums problem. Gray
  circles depict classification scores for each of the three action
  components (darker is higher), and colored outlines depict plausible
  temporal localizations. The three windows depicted here are red
  $(t = 1 \dots 4)$, blue $(t = 1 \dots 8)$, and green
  $(t = 3 \dots 8)$}
\label{fig:evolution2}
\vspace{-3mm}
\end{figure}

\begin{lemma}
\label{lem:rmax}
Let $\mathrm{rmax}_i[:]$ denote the list of $K$-best incomplete
temporal windows ending at timestep $i$, not including the end-point
$f^e[i]$. Namely, let
\begin{equation}
  \mathrm{rmax}_i[k]= \kmax_{j \in \{1, \cdots, i\}}
  \Big\{f^s(j)+\sum_{q = j+1}^{i}f^m(q) \Big\} \ .
\end{equation}
Then $\merge(f^s[i+1], \mathrm{rmax}_i[:]+f^m[i+1])$ gives the list of
the $K$-best incomplete temporal windows ending at timestep $i+1$.
\end{lemma}

\begin{proof}
  The $k$-th best incomplete window ending at frame $i+1$ is either
  the window that starts at $i+1$, or it is a continuation of one of
  the $K$-best windows that ended at frame $i$.
  $\mathrm{rmax}_i[:]+f^m[i+1]$ gives the list of all continuations of
  the previous $K$-best incomplete windows. We insert $f^s[i+1]$ to
  this list, and discard at most one of the $K$ continuations if
  $f^s[i+1]$ is greater than it. What remains are the $K$-best
  incomplete temporal windows ending at frame $i+1$.
\end{proof}

\begin{lemma}
  Let $\mathrm{kmax}_i[:]$ denote the list of the $K$ best temporal
  windows ending at or before frame $i$. Then
  $\merge(\mathrm{rmax}_{i}[:]+f^e[i+1], \mathrm{kmax}_{i}[:])$ gives
  the list of the $K$-best temporal windows ending at timestep $i+1$.
\end{lemma}

\begin{proof}
  We know from Lemma 1 that $\mathrm{rmax}_i[:]$ gives the $K$-best
  \emph{incomplete} temporal windows ending at frame $i$. The $k$-th
  best temporal window ending at frame $i+1$ is either one of these
  incomplete windows, completed by adding $f^e[i+1]$, or it is one of
  the top complete windows already contained in $\mathrm{kmax}_i$. By
  merging these two lists, we select the top-$K$ windows overall,
  preserving the top-$K$ complete temporal windows.
\end{proof}

\begin{algorithm}[t]
  \caption{Top-$K$ Structured Maximal Sums}
\label{alg:smekoptimal}
\begin{algorithmic}
  \REQUIRE Frame-wise scores $f^s[1 \cdots n]$, $f^m[:]$, and $f^e[:]$ \\
  \ENSURE $\mathrm{kmax}[1\cdots K]$ \FOR{each $k\leftarrow 1$ to $K$}
  \STATE $\textrm{kmax}[k]\leftarrow-\infty$, $\ \textrm{rmax}[k]\leftarrow-\infty$
\ENDFOR
\STATE $\textrm{rmax}[1]\leftarrow f^s[1]$ \hfill \COMMENT{Initialization}
\FOR {each $i\leftarrow 2$ to $n$}
  \FOR {each $k\leftarrow 1$ to $K$}
    \STATE $s \leftarrow \textrm{rmax}[k]+f^e[i]$
    \STATE $\textrm{rmax}[k]=\textrm{rmax}[k]+f^m[i]$
    \STATE $\textrm{kmax}[:] = \merge(s, \textrm{kmax}[:])$ \hfill
  \ENDFOR
  \STATE  $\textrm{rmax}[:] = \merge(f^s[i], \textrm{rmax}[:])$ 
\ENDFOR
\end{algorithmic}
\end{algorithm}
\vspace{-2mm}

Each $\mathrm{rmax}_i$ and $\mathrm{kmax}_i$ (including the call to
$\merge$) can be constructed in $\mathcal{O}(K)$ time
\cite{DBLP:conf/ispan/BaeT04}. We compute $\mathrm{kmax}_i$ for all
$i \in \{1, \dots, n\}$, so the total time complexity is
$\mathcal{O}(nK)$. This result, and the results of the above lemmas,
lead us to our primary theoretical contribution:

\begin{theorem}
\label{th:SME}
The SMS algorithm computes the $K$-best temporal windows in a video of
length $n$ in $\mathcal{O}(nK)$ time.
\end{theorem}

We note that, while this algorithm as written only computes the
\emph{scores} of the top-$K$ temporal windows, our implementation is
able to recover the windows themselves. This is accomplished with
simple bookkeeping which keeps track of the temporal windows' start-
and end-points as they are added to the $\mathrm{rmax}$ and
$\mathrm{kmax}$ lists.

\section{Training}

\begin{figure*}[t]
\centering
\begin{subfigure}{0.495\linewidth} \centering
  \includegraphics[width=0.95\textwidth]{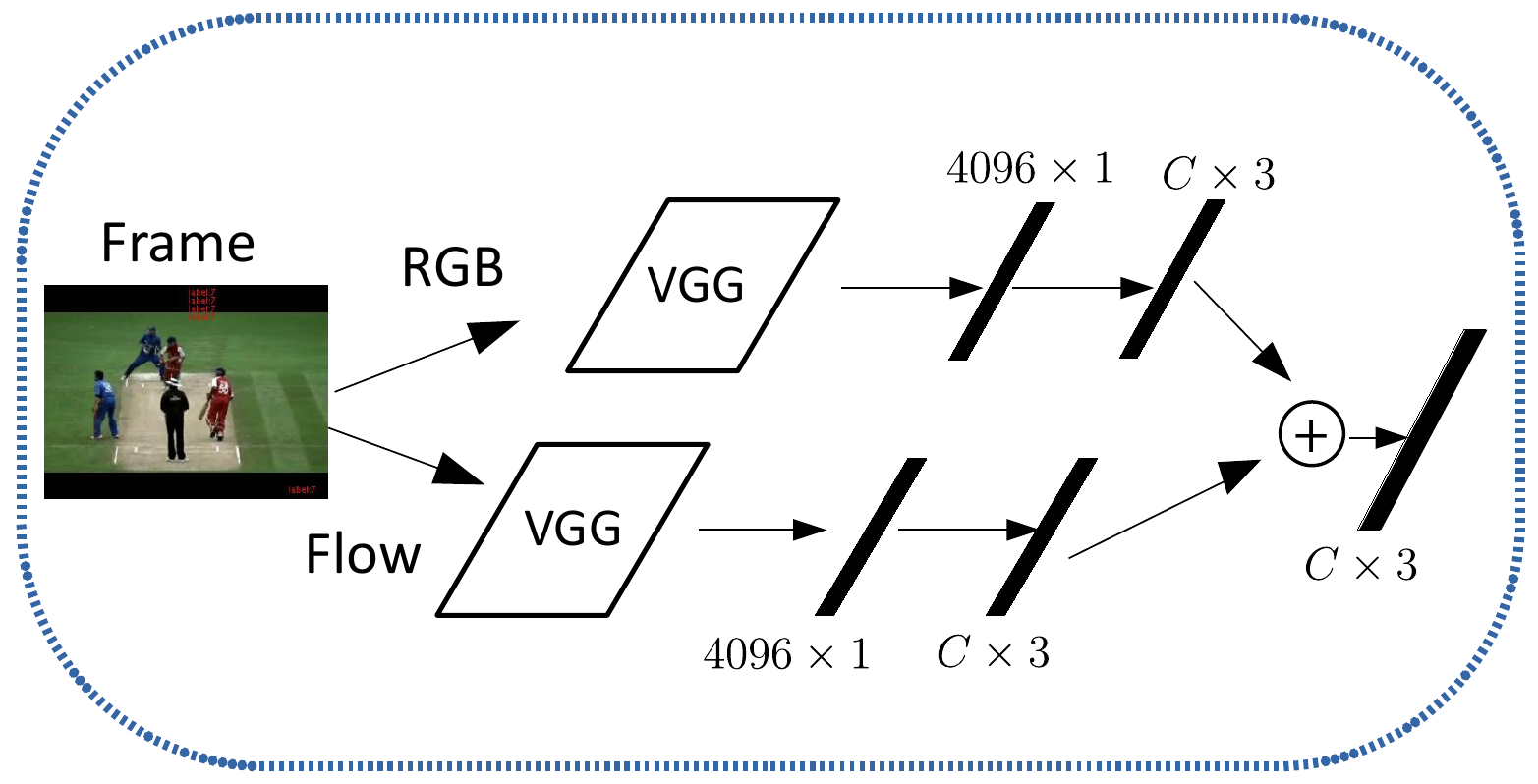}
\end{subfigure}
\begin{subfigure}{0.495\linewidth} \centering
  \includegraphics[width=0.95\textwidth]{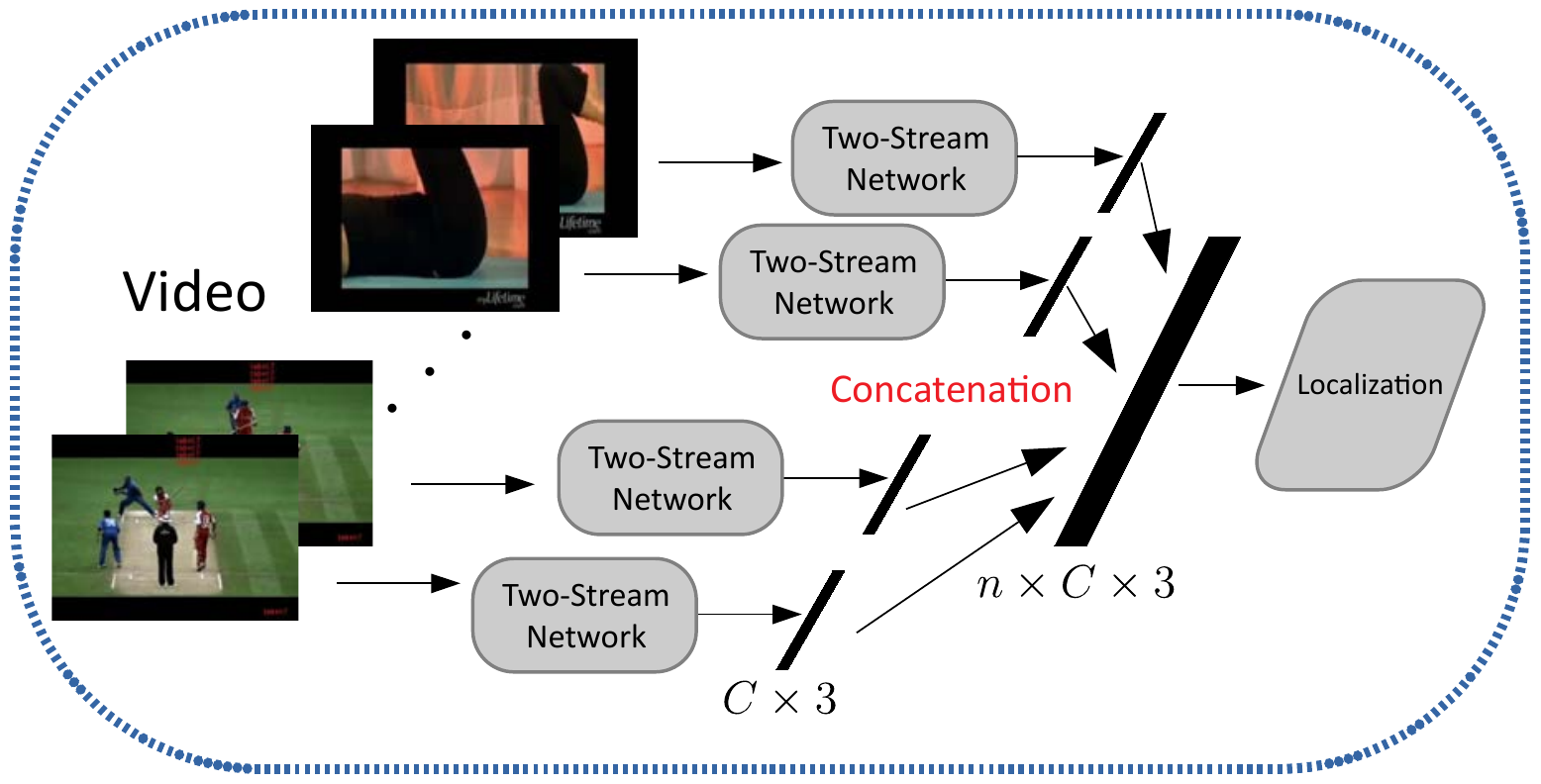}
\end{subfigure}
\vspace{-1mm}
\caption{Diagram of our localization architecture. (left) We use a
  two-stream network with a VGG backbone architecture to generate
  frame-wise confidence scores. (right) Scores from $n$ frames are
  concatenated and localization is performed.}
\vspace{-3mm}
\label{fig:architecture}
\end{figure*}

So far, we have assumed that all frame-wise action scores $f^s$,
$f^m$, and $f^e$ are computed beforehand. In this section, we describe
how these frame-wise score functions are learned. We use deep
Convolutional Neural Networks (CNNs) to produce the score functions,
and introduce a structured loss function which can be used to train
these CNNs in an end-to-end framework.

\subsection{Network Architecture}
\label{sec:features}

We adopt the two-stream network architecture of
\cite{DBLP:conf/nips/SimonyanZ14} to extract deep spatio-temporal
features for each video frame. The two-stream architecture consists of
two Convolutional Neural Networks (CNNs), \emph{Spatial-CNN} and
\emph{Motion-CNN}. The first network stream, \emph{Spatial-CNN},
operates on the color channels of a single video frame, capturing
features from the static appearance of the scene. The second stream,
\emph{Motion-CNN}, operates on dense optical flow fields computed
between adjacent frames, capturing distinctive motion features and
pixel trajectories over time. However, unlike
\cite{DBLP:conf/nips/SimonyanZ14}, we adopt the larger VGG 15-layer
network from \cite{DBLP:journals/corr/SimonyanZ14a} as the backbone
architecture for each of the two streams. For each stream, we produce
$(C\times 3)$-dimensional outputs, where $C$ is the total number of
action classes. Finally, we average the frame-wise scores from the two
streams and concatenate the results across frames. This architecture
is pictured in Figure~\ref{fig:architecture}.

\subsection{Structured Loss}
\label{sec:structuralsvm}

Our model minimizes a video-level structured loss function, rather
than the frame-wise loss function used to train the typical two-stream
action recognition architecture. By directly optimizing for temporal
action localization, we enable the frame-wise scores to take into
account the temporal evolution of actions. This enables a level of
fine-tuning that would not be possible by optimizing for a frame-wise
objective. As in the typical two-stream architecture, we first
pre-train each stream separately and fuse by finetuning.

We have a dataset $V = \{v_1,v_2,\cdots,v_m\}$ of $m$ training videos
and labels $Y=\{y_1, y_2, \cdots, y_m\}$. Each video
$v_i = \{x^{(i)}_1, \cdots, x^{(i)}_{n_i}\} \in \mathcal{V}$ can be
arbitrarily long, and its length is denoted $n_i$. For simplicity, we
assume the training videos contain only a single action instance.
These labels $y_i = (s^{(i)}, e^{(i)}, \ell^{(i)}) \in \mathcal{Y}$
consist of a start index $s$, end index $e$, and action label $\ell$.
Our goal is to learn a confidence function
$F: \mathcal{V} \times\mathcal{Y} \mapsto\mathbb{R}$, which measures
how likely it is that a particular action instance is present in the
video. We require $F$ to take on the frame-wise summation form as in
Equation~\ref{equa:confidence2}. We denote the learnable parameters of
$F$, namely those of the CNNs, as $\mathbf{w}$. We use the notation
$F(v,y;\mathbf{w})$ to denote the confidence score produced for a
video $v$ and window $y$ with parameters $\mathbf{w}$. 

For a training video $v_i$, we define the localization loss
$\mathcal{L}_{loc}$ as the gap between the highest-scoring temporal
window and the ground truth label for the action $\ell^i$:
\begin{equation}
  \mathcal{L}_{loc} (v_i) = \bigg[\max_{y\neq y_i}\big\{\Delta(y_i,y)+F(v_i,y;\mathbf{w})\big\}-F(v_i,y_i;\mathbf{w})\bigg]_+,
\end{equation}
where $[\cdot]_+ = \max(0, \cdot)$ is the hinge loss function
\cite{DBLP:conf/nips/GentileW98}. The $\Delta$ term is added to weaken
the penalty on windows with high overlap with ground truth, and is
defined as
$\Delta(y,\overline{y})=|y\cup\overline{y}|-|y\cap\overline{y}|$,
where each predicted window $y$ and $\overline{y}$ is a set of video
frames and $|\cdot|$ is the cardinality.

To further make the network more discriminative, we introduce a
classification loss $\mathcal{L}_{cls}$, which enforces that the
estimated windows of other actions should have lower scores than those
of the ground truth action class. We define
\begin{equation}
  \mathcal{L}_{cls}(v_i)= \frac{1}{C-1}\big[M+\max_{y: \  \ell\neq \ell^i}F(v_i,y;\mathbf{w})-F(v_i,y_i;\mathbf{w})\big]_+,
\end{equation}
where $M$ is a fixed parameter that ensures we do not penalize the
detection if the distance is already lower than $M$. In our
experiments, we set $M$ to be the ground truth window length $|y_i|$
of the video $v_i$.

The full structured objective $\mathcal{L}$ is a sum of the two losses
over all videos in the training set, defined as follows:
\begin{equation}
\mathcal{L}(V)=\sum_{i=1}^{m}\big(\mathcal{L}_{loc}(v_i)+\lambda\mathcal{L}_{cls}(v_i)\big)
\ ,
\end{equation}
where $\lambda$ weights the relative importance of the two loss
functions. By default, we set $\lambda = 0.5$.

Note that both loss functions are typical structural SVM losses and
the parameters $\mathbf{w}$ can therefore be learned in a similar
end-to-end fashion \cite{DBLP:journals/mp/Shalev-ShwartzSSC11}. Since
both the localization and classification losses are
sub-differentiable, the parameters of the two CNN streams can be
learned by backpropagation. Specifically, for one layer $l$, the
gradient of either loss on one video $\mathcal{L}_{(\cdot)}(v_i)$ with
respect to that layer's parameters, $w^{(l)}$, can be calculated as
\begin{align}
\begin{split}
\label{equa:dloss}
\frac{\partial \mathcal{L}_{(\cdot)}(v_i)}{\partial \mathbf{w}^{(l)}}
= \bigg(&\frac{\partial \mathcal{L}_{(\cdot)}(v_i)}{\partial
  F(v_i,y_i^*)}\frac{\partial F(v_i,y_i^*)}{\partial \bf{f}} \\
&-\frac{\partial \mathcal{L}_{(\cdot)}(v_i)}{\partial
  F(v_i,y_i)}\frac{\partial F(v_i,y_i)}{\partial \bf{f}}\bigg)
\frac{\partial \bf{f}}{\partial \bf{w}^{(l)}}
\end{split}
\end{align}
where $y_i^*$ represents $\argmax_y (\Delta(y,y_i)+F(v_i,y))$ for
$\mathcal{L}_{loc}$ and $\argmax_{y,l\neq \ell^i} F(v_i,y)$ for
$\mathcal{L}_{cls}$. $\bf{f}$ is the set of frame-wise confidence
scores produced by the neural networks.

\begin{algorithm}[t]
\caption{Loss-Augmented Stuctured Maximal Sum}
\label{alg:lossaugmented}
\begin{algorithmic}
  \REQUIRE Confidence scores $f^s[1 \cdots n]$, $f^m[:]$ and $f^e[:]$;
  ground truth window $y=\{s, e\}$\\
  \ENSURE $\mathrm{smax}$ \STATE $\mathrm{smax}\leftarrow-\infty$;
  $\mathrm{rsum}[1]\leftarrow-\infty$; $p \leftarrow 0$ \hfill
  \COMMENT{Initializaition} \FOR{each $i\in [1,s)\cup (e,n]$} \STATE
  $f^s[i]\leftarrow f^s[i]+1$; $f^m[i] \leftarrow f^m[i]+1$;
  $f^e[i]\leftarrow f^e[i]+1$;
\ENDFOR
\FOR {each $j\leftarrow 2$ to $N$}
\STATE $len\leftarrow\max [0,\min(e-s+1,e-j)]$
\STATE $\mathrm{rsum}[j]\leftarrow\max(\mathrm{rsum}[j-1]+f^m[j],p+f^s[j])$
\STATE $\mathrm{smax}\leftarrow\max(\mathrm{smax},\mathrm{rsum}[j-1]+len+f^e[j])$
\IF{$j\in[s,e]$}
\STATE $p\leftarrow p+1$
\ENDIF
\ENDFOR
\end{algorithmic}
\end{algorithm}

The gradient of $\bf{f}$ with respect to the network parameters,
$\frac{\partial \bf{f}}{\partial \bf{w}^{(l)}}$, can be computed via
backpropagation, as is standard for CNNs without the structured
objective. Therefore, it remains to compute two gradients: (1) the
gradient of the confidence function $F$ w.r.t. the classifiers
$\bf{f}$ and (2) the gradient of the objective function
$\mathcal{L}_{(\cdot)}$ w.r.t. $F$. To compute (1), we recall that
confidence function is simply the summation of the action parts
scores, so its gradient computation is straightforward. Although (2)
is not differentiable, it is in fact sub-differentiable, so we compute
a subgradient:
\begin{equation}
  \label{equa:subgradient}
\frac{\partial \mathcal{L}}{F(v_i,y_i^*)}=
\begin{cases}
  1& \text{if } \Delta(y_i,y_i^*)+F(v_i,y_i^*)-F(v_i,y_i)> 0\\
  0& \text{otherwise}.\\
\end{cases}
\end{equation}
This allows us to train end-to-end with subgradient descent.

To compute the subgradient, we need to find the best window $y_i^*$.
We use the SMS algorithm (Algorithm~\ref{alg:smekoptimal}) to find
$y_i^*$ in $\mathcal{L}_{cls}$. However, because of the $\Delta$ term
in the maximization in $\mathcal{L}_{loc}$, in order to compute
$y_i^*$ we are required to perform a maximization of the
\emph{loss-augmented confidence}. In
Algorithm~\ref{alg:lossaugmented}, we modify the SMS algorithm to
include this term, achieving the same linear time complexity,
guaranteeing that this can be computed efficiently during training.
Additionally, we only compute the top detection.

\section{Experiments}
\label{sec:experiments}

We evaluate our method on the THUMOS'14 dataset
\cite{THUMOS15}. Our implementation is built on Caffe
\cite{DBLP:conf/mm/JiaSDKLGGD14}.

\noindent {\bf Two-stream neural networks.} The inputs to Spatial-CNN
are RGB video frames cropped to $224 \times 224$ with the mean RGB
value subtracted. The inputs to Motion-CNN are dense optical flow
channels computed by the TVL1 optical flow algorithm
\cite{DBLP:conf/dagm/ZachPB07}. We scale each optical flow image to be
between $[1,255]$ and stack the flows of $10$ frames in both
directions to form a $224\times224\times20$ 3D volume. Spatial-CNN is
pre-trained for object recognition on ImageNet \cite{ILSVRC15}, and
Motion-CNN is pre-trained for action classification on UCF101
\cite{soomro2012ucf101}. We train Spatial-CNN and Motion-CNN
separately, then jointly fine-tune both of their final two
fully-connected layers. Additionally, we adopt multi-scale random
cropping for both streams. For each sample, we first randomly choose a
scale from a predefined list, then choose a random crop of size
$(224 \times 224) \times scale$. The cropped region is resized to
$224 \times 224$ before being input into the network. For Spatial-CNN,
three scales $[1,0.875,0.75]$, for Motion-CNN, we use two scales
$[1,0.875]$.

\begin{figure*}[]
\centering
\includegraphics [width=.96\textwidth] {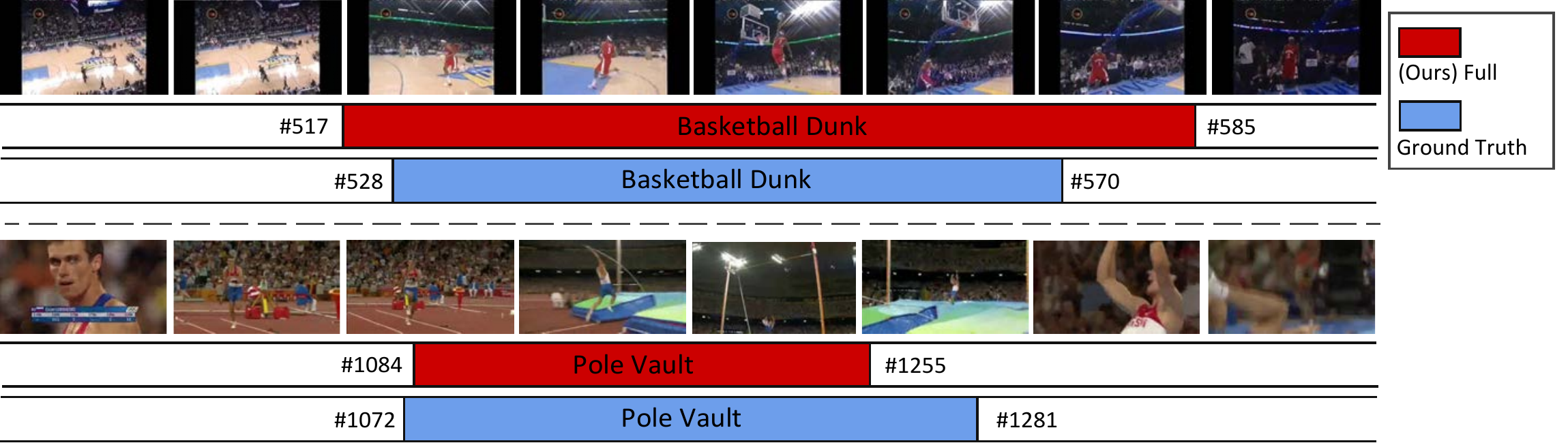}
\caption{Example detections from our system. Frame numbers are given at
  the boundaries of each action.}
\label{fig:example1}
\end{figure*}

\noindent {\bf Postprocessing.} We divide each testing video into
overlapping 20-second snippets with an 18-second overlap between
neighboring snippets, and perform localization independently for each
snippet. Subsequently, we merge predictions across these snippets. We
set the number of action instances to $K = 100$, as experimentally
$K\ge 100$ does not improve recall on our validation set. The temporal
action window scores from Equation~\ref{equa:confidence2} are prone to
giving higher scores to longer windows, so we additionally normalize
the confidence of each window by its window length. Furthermore, we
multiply confidence scores by action duration priors as in
\cite{DBLP:conf/iccv/OneataVS13} to encourage action windows to have
reasonable lengths. After generating all candidates, we filter those
with large overlap using non-maximum suppression.

\noindent {\bf Balanced training.} Middle frames are more prevalent
than start and end frames, so to prevent the network from becoming
biased towards middle frames, we divide each middle frame's score by
the total window length during training. In addition, since the manual
annotations of start and end of actions are relatively noisy, we
randomly sample the start frame from the first $10\%$ of frames and the
end from the last $10\%$. Middle frames are sampled from the middle
$80\%$.

\noindent {\bf Evaluation Metric.} We use mean Average Precision (mAP)
to measure localization performance as in \cite{THUMOS15}. We count a
detection as correct if it predicts the correct action label and its
intersection over union (IOU) with ground truth is larger than some
overlap threshold $\sigma$.

\begin{figure*}[]
\centering
\includegraphics [width=\textwidth] {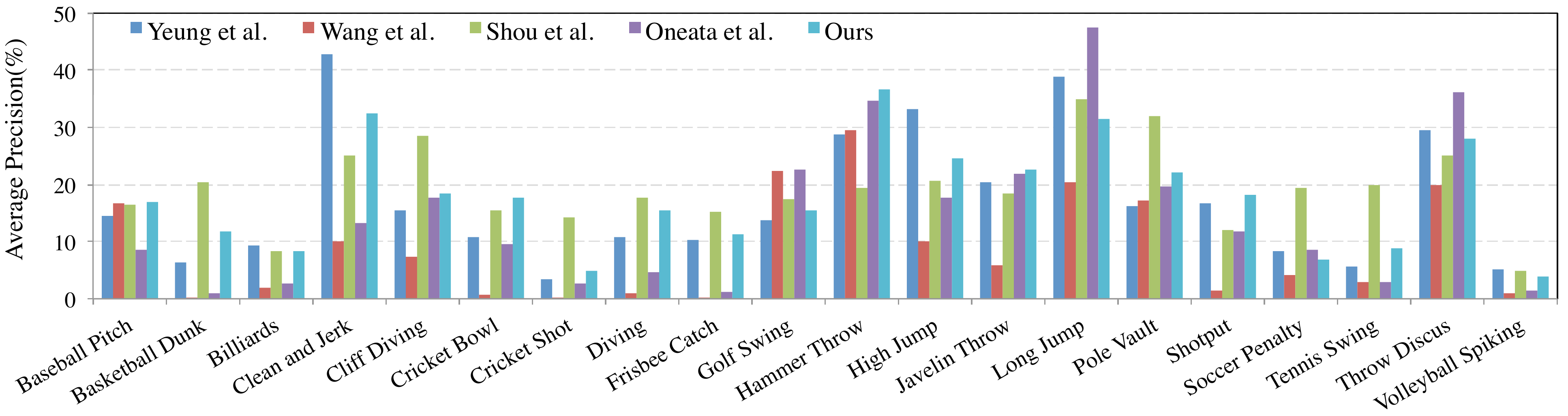}
\vspace{-6mm}
\caption{Per-class average precision on THUMOS '14 at overlap
  threshold $\sigma=0.5$.}
\label{fig:class1}
\vspace{-4mm}
\end{figure*}

\begin{figure*}[]
\centering
\includegraphics [width=.96\textwidth] {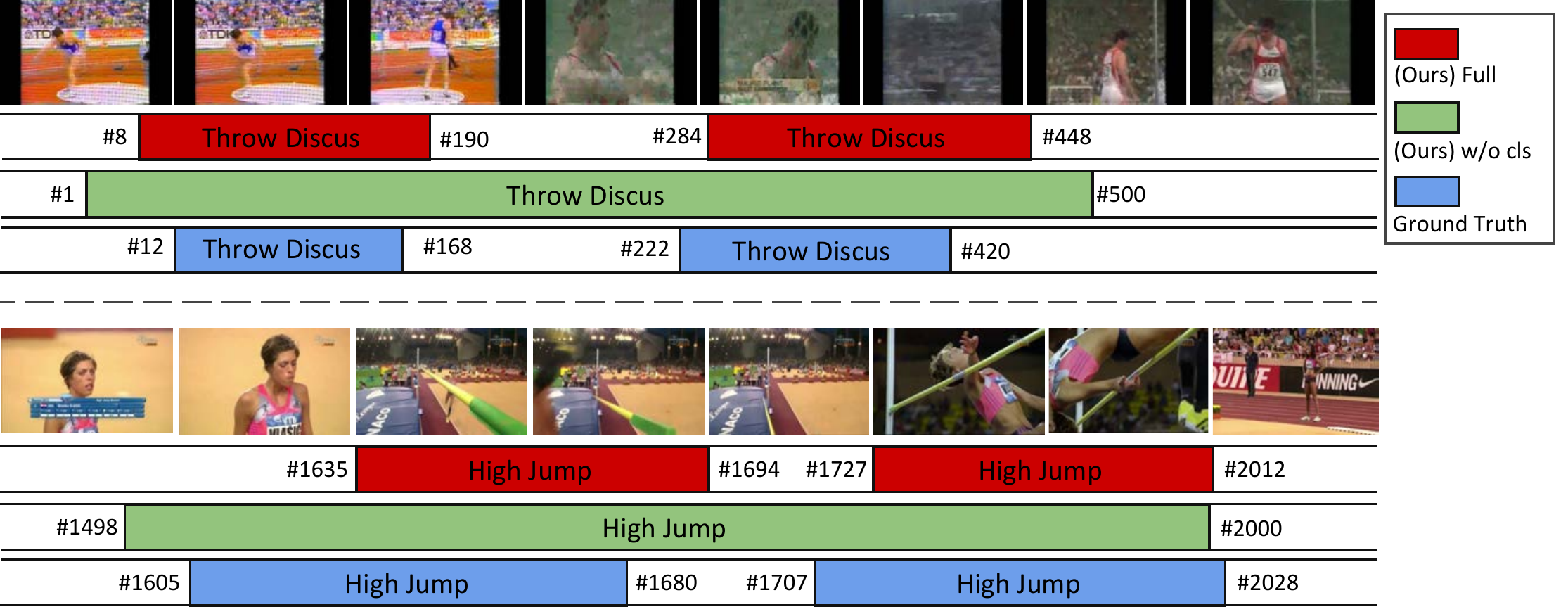}
\caption{Example detections from our full system compared with our
  system trained without the classification loss $\mathcal{L}_{cls}$.}

\label{fig:example2}
\end{figure*}

\subsection{Results on THUMOS 2014}

THUMOS '14 includes 20 sports action classes for temporal detection,
with over 13K short clips and background videos for training, 1010
untrimmed validation videos, and 1574 untrimmed videos for testing. In
our experiments, we use the training clips, background, and validation
videos for training, and report results on the untrimmed test videos.
During training, we crop each validation video to short 800-frame
clips which contain one single action instance. We augment the
training dataset by splicing action instances from the training and
validation clips with background videos and validation videos in which
no instances of the 20 classes appear. In total, we generate 42000
action clips for training. We choose hyperparameters based on results
on a withheld subset of the validation videos. We train Spatial-CNN
and Motion-CNN for 16K and 20K iterations, respectively. We then
finetune the two-stream network for 2K additional iterations. At test
time, we downsample all videos to 5fps, and filter out videos that are
unlikely to contain any of the 20 action classes. We do this by
averaging their frame-level class scores from action recognition
models \cite{DBLP:journals/corr/WangXW015} finetuned on THUMOS'14. In
Figure~\ref{fig:example1}, we show example detections on the THUMOS'14
test set.

\begin{table}[]
\centering
\begin{tabular}{l|c|c|c|c|c}
  \multicolumn{6}{c}{Comparison with State-of-the-Art} \\
  \toprule \hline 
  overlap threshold $\sigma$                               & 0.1           & 0.2           & 0.3           & 0.4           & 0.5                   \\ \hline  \bottomrule  
  Karaman {\it et al.}~\cite{Karaman}                      & 1.5           & 0.9           & 0.5           & 0.3           & 0.2                   \\ \hline
  Wang {\it et al.}~\cite{wang2014action}                  & 19.2          & 17.8          & 14.6          & 12.1          & 8.5                   \\ \hline
  Oneata {\it et al.}~\cite{DBLP:conf/iccv/OneataVS13}     & 39.8          & 36.2          & 28.8          & 21.8          & 15.0                  \\ \hline
  Shou {\it et al.}~\cite{DBLP:journals/corr/ShouWC16}    & 47.7          & 43.5          & 36.3          & \textbf{28.7} & \textbf{19.0}         \\ \hline
  Yeung {\it et al.}~\cite{DBLP:journals/corr/YeungRMF15}  & 48.9          & 44.0          & 36.0          & 26.4          & 17.1                  \\ \hline
  Richard {\it et al.}~\cite{DBLP:journals/cvpr/Alexander16}& 39.7         & 35.7          & 30.0          & 23.2          & 15.2                  \\ \hline
  Ours (full)                                              & \textbf{51.0} & \textbf{45.2} & \textbf{36.5} & 27.8          & 17.8                  \\ \hline
  \bottomrule
\end{tabular}
\vspace{-3mm}
\caption{The mean average precision (mAP) of different methods for varying overlap thresholds. Our system achieves state-of-the-art performance for $\sigma=0.1, 0.2, 0.3$.}
\label{tab:comparison}
\vspace{-4mm}
\end{table}

\begin{table}[t!]
\centering
\begin{tabular}{l|c|c|c|c|c}
  \multicolumn{6}{c}{Ablation Study} \\
  \toprule \hline
  overlap threshold $\sigma$                               & 0.1           & 0.2           & 0.3           & 0.4           & 0.5                   \\ \hline \bottomrule
  Baseline                                                 & 18.5          & 10.2          & 4.5           & 1.8           & 0.2                   \\ \hline
  \textit{w/o cls + sme}                                   & 40.5          & 28.8          & 23.2          & 16.4          & 13.2                  \\ \hline
  \textit{w/o cls}                                         & 42.5          & 32.6          & 27.8          & 19.6          & 15.7                  \\ \hline
  \textit{w/o sme}                                         & 48.0          & 42.2          & 33.0          & 24.8          & 16.2                \\ \hline
\textit{w/o prior} & 50.7 & 45.0 & 36.2 & 27.4 & 17.5
\\ \hline
    Ours (full)                                            & \textbf{51.0} & \textbf{45.2} & \textbf{36.5} & \textbf{27.8} & \textbf{17.8}        \\ \hline \bottomrule
  \multicolumn{6}{c}{Separate networks}  \\
  \toprule \hline                       
  \textit{spatial}                                         & 46.2          & 40.3          & 31.5          & 23.2          & 16.0  \\ \hline
  \textit{motion}                                          & 47.6          & 44.0          & 35.6          & 25.8         & 16.9  \\ \hline
  \textit{late fusion}                                     & 46.0          & 43.2          & 32.8          & 24.0          & 14.5  \\ \hline
    Ours (full)                                            & \textbf{51.0} & \textbf{45.2} & \textbf{36.5} & \textbf{27.8} & \textbf{17.8}        \\ \hline \bottomrule
  \bottomrule
\end{tabular}
\vspace{-3mm}
\caption{Ablation experiments for the structured objective (top) and the two-stream architecture (bottom). The full model outperforms all other configurations.}
\label{tab:ablation}
\vspace{-6mm}
\end{table}

We report results at varying overlap thresholds in
Table~\ref{tab:comparison} and compare with existing systems. Our
model outperforms state-of-the-art when the overlap threshold $\sigma$
is $0.1$, $0.2$ and $0.3$, and achieves competitive results for $0.4$
and $0.5$. This indicates that our system can distinguish action
instances from background frames even when precise localization is
difficult. Additionally, we provide per-class average precision
results in Figure \ref{fig:class1}. Our system achieves the best
performance on 5 of the 20 actions. For a few actions, namely
\emph{Billards}, \emph{Cricket Shot}, \emph{Tennis Swing} and
\emph{Volleyball Spiking}, we get a relatively low average precision.
This could be due to two main reasons: (1) these action instances are
short and thus the action scores are relatively noisy compared to longer
actions and (2) these actions often occur in rapid succession, making
it easy to merge adjacent action instances.

\noindent {\bf Ablation Study} In order to show the contribution of
each component in our system, we experiment with eight variants of the
full pipeline. The results are reported in Table~\ref{tab:ablation}.

In \textit{baseline}, we do not train using a structured loss function
during training. Instead, we train our model to perform mutually
exclusive framewise classification, and subtract the mean confidence
to form signed confidence scores, and perform localization using SMS.
In \textit{w/o cls}, we drop the structured localization loss
$\mathcal{L}_{cls}$ during training. In \textit{w/o sme}, we do not
model start-middle-end temporal evolution, and each action score is
computed as the average of frame-level action classification scores.
In \textit{w/o cls + sme}, we use drop $\mathcal{L}_{cls}$ and also do
not model temporal evolution. Dropping each of these components
results in a significant drop in performance, indicating that both
temporal evolution and the structured classification loss are
important to facilitate training and accurate localization. In
\textit{w/o prior}, we drop the action duration priors, which results
in a small drop in performance. In Figure~\ref{fig:example2}, we give
examples of detections produced when the $\mathcal{L}_{cls}$ loss is
dropped, localization becomes less precise. We also perform a separate
ablation study on the components of the two-stream architecture. We
evaluate using each stream individually (\emph{spatial} and
\emph{motion}), as well as simply averaging the two streams rather
than fine-tuning jointly (\textit{late fusion}). We find that the full
model outperforms \textit{late fusion}, suggesting the importance of
joint training of the two streams. The \textit{late fusion} does not
outperform separate networks, suggesting an incompatibility of
confidence scores from the two networks.

\subsection{Conclusions}

We present a framework for end-to-end training of temporal
localization that takes into account the \emph{temporal evolution} of
each action. We frame localization as a \emph{Structured Maximal Sum}
problem, and provide efficient algorithms for training and detection
in this framework. We show that modeling temporal evolution improves
performance, and demonstrate that our system achieves competitive
performance on the THUMOS '14 benchmark.

\subsubsection*{Acknowledgements}
This work is partially supported by a University of Michigan graduate
fellowship, the Natural Science Foundation of China under Grant No.
61672273, No. 61272218, and No. 61321491, and the Science Foundation
for Distinguished Young Scholars of Jiangsu under Grant No.
BK20160021.

{\small
\bibliographystyle{ieee}
\bibliography{mybib}
}

\end{document}